\documentclass[10pt,conference]{IEEEtran}
\usepackage[bookmarks=true]{hyperref}
\usepackage{times}

\usepackage[numbers]{natbib}
\usepackage{multicol}

\usepackage{amsmath,amssymb,amsfonts}
\usepackage[ruled, vlined]{algorithm2e}
\usepackage{graphicx}
\usepackage{epstopdf}
\usepackage{textcomp}
\usepackage{xcolor}

\usepackage{siunitx}

\usepackage{amsthm}
\newtheorem{defn}{Definition}
\newtheorem{thm}{Theorem}

\usepackage{mymacros}

\newcommand{\eshadow}{$\epsilon$-shadow}
\newcommand{\ejshadow}{$\epsilon_j$-shadow}

\begin{document}

\title{
    Fast Certification of Collision Probability Bounds with Uncertain Convex Obstacles
}

\author{Author Names Omitted for Anonymous Review. Paper-ID 106}

\author{\IEEEauthorblockN{Charles Dawson}
\IEEEauthorblockA{\textit{Dept. of Aeronautics and Astronautics} \\
\textit{Massachusetts Institute of Technology}\\
Cambridge, MA, USA \\
cbd@mit.edu}
\and
\IEEEauthorblockN{Andreas Hofmann}
\IEEEauthorblockA{\textit{Dept. of Aeronautics and Astronautics} \\
\textit{Massachusetts Institute of Technology}\\
Cambridge, MA, USA \\
hofma@csail.mit.edu}
\and
\IEEEauthorblockN{Brian Williams}
\IEEEauthorblockA{\textit{Dept. of Aeronautics and Astronautics} \\
\textit{Massachusetts Institute of Technology}\\
Cambridge, MA, USA \\
williams@csail.mit.edu}
}

\maketitle

\begin{abstract}

To operate reactively in uncertain environments, robots need to be able to quickly estimate the risk that they will collide with their environment. This ability is important for both planning (to ensure that plans maintain acceptable levels of safety) and execution (to provide real-time warnings when risk exceeds some threshold). Existing methods for estimating this risk are often limited to models with simplified geometry (e.g. point robots); others handle complex geometry but are too slow for many applications. In this paper, we present two algorithms for quickly computing upper bounds on the risk of collision between a robot and uncertain obstacles by searching for certificate regions that capture collision probability mass while avoiding the robot. These algorithms come with strong theoretical guarantees that the true risk does not exceed the estimated value, support arbitrary geometry via convex decomposition, and provide fast query times ($<$\SI{200}{\micro s}) in representative scenarios. We characterize the performance of these algorithms in environments of varying complexity, demonstrating at least an order of magnitude speedup over existing techniques.

\end{abstract}

\section{Introduction}

To operate safely in the real world, robots must be able to manage risk stemming from the pervasive uncertainty that distinguishes real environments from carefully-managed laboratory tests. Outside the lab, robots must contend with factors such as sensor noise and human unpredictability that increase the risk of collision and injury. In order to manage these various sources of uncertainty, robots need to be able to quickly and accurately estimate the probability that a given configuration will result in a collision. For example, if a robot seeks to plan a trajectory where the risk of collision does not exceed some threshold, then it must have some way to evaluate the risk of collision at each step along the trajectory. Similarly, if an autonomous vehicle must maintain a set level of safety or else revert to a fail-safe state, then it must continuously track the probability of imminent collision.

In safety-critical applications, where false-negatives (underestimates of true risk) can have severe consequences, we especially desire measures that are guaranteed to never underestimate the risk of collision. By providing such guarantees (through the use of easily-verified risk certificates), we can limit the frequency of false-negatives: an autonomous car using such conservative estimates might ask for assistance more often, but it would never stay silent in a dangerous situation.

Real-world environments like factories and roads are challenging to navigate not only because they involve uncertainty but also because they involve obstacles with complex geometry (e.g. shelves or debris). Unfortunately, many existing approaches for computing collision risk in uncertain environments rely on simplified geometric representations, representing robots as points \cite{blackmoreChanceConstrainedOptimalPath2011,axelrodProvablySafeRobot2018,ludersChanceConstrainedRRT2010} or spheres \cite{vandenbergLQGMPOptimizedPath2011,parkFastBoundedProbabilistic2018}, rather than supporting arbitrary geometries.
Of the approaches that support complex geometries, many consider uncertainty only in the state of the robot \cite{daiChanceConstrainedMotion2018,sunSafeMotionPlanning2016}, ignoring potential uncertainty in the location of obstacles. These approaches may work well in static environments, but when relying on noisy sensors in dynamic environments one must account for uncertainty in the environment as well. For example, when designing a robotic arm, it is much easier to add sensors that track the robot's joint angles than to instrument the entire environment, so obstacle uncertainty dominates uncertainty in robot state. Approaches that support uncertain obstacles in addition to complex geometry exist but require query times between \SI{100}{ms} and \SI{10}{s} \cite{jasourMomentSumofSquaresApproachFast2018,parkEfficientProbabilisticCollision2017}, which can be disqualifying in safety-critical applications. An autonomous vehicle traveling at \SI{20}{m/s} on a busy street cannot tolerate even a \SI{100}{ms} delay.

Because of these gaps, there is a need to develop methods for calculating the probability of collision between a robot and its environment that a) support complex robot and environment geometry, b) account for uncertainty in the location of obstacles, and c) minimize computation time. In this paper, we specifically consider the case where the robot and obstacles are both represented as convex shapes, so that arbitrarily complex shapes can be represented as collections of convex sub-shapes, and where obstacles are subject to Gaussian uncertainty in their location.

\subsection{Contributions}

To address these needs, we present two algorithms that estimate the probability of collision between uncertain convex objects by finding certificates proving that the collision risk is below some bound. These certificates take the form of regions capturing a certain amount of collision risk; by avoiding those regions the robot limits its exposure to the captured risk.

The first algorithm uses computationally-efficient techniques from convex geometry to produce certificates of collision risk for robots and environments with non-trivial geometry This algorithm is based on previous approaches to generating collision-risk certificates for point robots, but extends those approaches to non-trivial geometry.

The second algorithm makes use of a novel multi-step search process to expand certificate regions into unoccupied areas of the environment, generating significantly tighter upper bounds on collision probability while sacrificing only a factor of 2 run-time penalty compared with the first algorithm.

Both algorithms provide strong theoretical guarantees that the true probability of collision does not exceed the estimate, both handle complex robot and environment geometry, and both scale linearly with respect to the number of obstacles and the number of robot links present.

\section{Related Work}

In the absence of uncertainty, robots can make use of a number of mature computational geometry packages for detecting collisions; prominent examples include \texttt{libccd} \cite{fiserLibccd2019}, \texttt{fcl} \cite{panFCLGeneralPurpose2012}, and the Bullet collision library \cite{coumansBulletPhysicsEngine}. Using the Gilbert-Johnson-Keerthi (GJK) \cite{gilbertFastProcedureComputing1988} and expanding polytope (EPA) algorithms \cite{bergenFastRobustGJK1999}, modern libraries can check collisions between convex shapes in microseconds.

In contrast, state-of-the-art algorithms for estimating the probability of collision between uncertain convex shapes have yet to attain similar levels of performance. Of particular note is the approach proposed by Park, Park, and Manocha \cite{parkEfficientProbabilisticCollision2017}. This approach computes the probability of collision between convex shapes by first computing the Minkowski sum of the two shapes, then computing an approximate integral over the faces of the summed shape. This approach demonstrates impressive accuracy and a novel iterative decomposition approach to convexifying non-convex geometry, but its performance is limited by expensive geometric operations such as explicitly computing set-wise sums and integrating over the faces of a 3D mesh. As a result, the authors report query times exceeding \SI{100}{ms} for convexified shapes. To avoid these expensive operations, our approach relies on a support vector representation of convex geometry, which allows constant-time construction of implicit Minkowski sums (compared to $O(v^2)$ time for explicit construction from shapes with $v$ vertices), and we avoid integration altogether, relying only on $O(v)$ GJK collision checking. For an introduction to support vector geometry, the reader is referred to \cite{schulmanFindingLocallyOptimal2013}.

A popular alternative approach to computing collision risk has been to consider not the nominal geometry but inflated shapes that represent confidence intervals encompassing space where the robot or obstacle is likely to be \cite{blackmoreChanceConstrainedOptimalPath2011}. Lee et al. account for uncertainty in robot state by generating convex confidence intervals of robot geometry \cite{leeSigmaHullsGaussian2013}, and Axelrod, Kaelbling, and Lozano-P\'erez have developed a method for expanding uncertain obstacles into (non-convex) confidence intervals \cite{axelrodProvablySafeRobot2018}. These approaches reduce the uncertain path-planning problem to deterministic path planning around expanded obstacles; however, there are a number of gaps in these techniques. The approach of Lee et al. does not provide theoretical guarantees bounding the risk of collision, and although the approach of Axelrod, Kaelbling, and Lozano-P\'erez provides such guarantees, it requires non-convex geometry that cannot be efficiently checked for collision, limiting its performance. Furthermore, in their treatment, Axelrod, Kaelbling, and Lozano-P\'erez consider only point robots, limiting its applicability.

Other approaches to estimating collision probabilities include sampling-based methods \cite{daiChanceConstrainedMotion2018,blackmoreProbabilisticParticleControl2006}, analytic approaches relying on point-mass robots and linear- or polynomial-inequality obstacles \cite{blackmoreChanceConstrainedOptimalPath2011,ludersChanceConstrainedRRT2010,jasourMomentSumofSquaresApproachFast2018}, and configuration-space collision checking, where the robot can indeed be treated as a point but obstacles are often non-convex \cite{sunSafeMotionPlanning2016}.

\section{Preliminaries}

In this section, we introduce notation, define the problem statement, and prove a theorem used in the rest of this paper.

\subsection{Problem Statement}

In the following discussion, upper case script symbols (e.g. $\mathcal{X}$, $\mathcal{O}$) are used to denote subsets of $\R^n$, such as the set of points occupied by one link of a robot or the set of points occupied by an obstacle. We primarily consider the problem of collision checking in the 3-dimensional workspace, but our approach can be generalized easily to higher dimensions. Furthermore, we restrict our analysis to the case when all shapes are convex, as most complex geometries can be represented in practice using a convex decomposition.

Given a collection of convex shapes $\mathcal{X}_1, \ldots, \mathcal{X}_m$, representing convex links of a robot, and a set of obstacles $\mathcal{O}_1, \ldots, \mathcal{O}_k$ with known shape but uncertain location, we seek to compute an upper bound $\epsilon$ on the probability that any robot link intersects any obstacle: 
\begin{equation}
    P\pn{\bigcup_{(i, j)} \mathcal{X}_i \cap \mathcal{O}_j \neq \emptyset} \leq \epsilon
\end{equation} 

The ability to quickly compute this probability bound is important in the context of risk-aware motion planning, where collision probability estimation is often a bottleneck, and in safety verification, which requires real-time performance.

This probability bound can be computed by applying Boole's inequality and calculating bounds on the probabilities that individual obstacles collide with the robot. Let $\epsilon_j$ be an upper bound on the probability that obstacle $\mathcal{O}_j$ collides with any link of the robot:
\begin{equation}
    P\pn{\bigcup_{i} \mathcal{X}_i \cap \mathcal{O}_j \neq \emptyset} \leq \epsilon_j
\end{equation}
Then the probability that \textit{any} obstacle collides with the robot is bounded above by $\sum_j \epsilon_j$, by Boole's inequality:

\begin{align}
    P\pn{\bigcup_{j}\bigcup_{i} \mathcal{X}_i \cap \mathcal{O}_j \neq \emptyset} &\leq \sum_j P\pn{\bigcup_{i} \mathcal{X}_i \cap \mathcal{O}_j \neq \emptyset} \label{eq:boole}
\end{align}

In subsequent sections, we show how these individual $\epsilon_j$ bounds can be computed efficiently using existing convex collision checking algorithms.

\subsection{Shadows of Uncertain Obstacles}

Previous work by Axelrod et al. has used geometric objects known as \eshadow s (defined below) to characterize uncertain obstacles \cite{axelrodProvablySafeRobot2018}. In this and subsequent sections, we use $\mathcal{O}$ to refer to an arbitrary obstacle.

\begin{defn}\label{eshadow_defn}
    {\normalfont (\eshadow)} A set $\mathcal{S} \subseteq \R^n$ is an \eshadow \ of an uncertain obstacle $\mathcal{O}$ if the probability $P(\mathcal{O} \subseteq \mathcal{S}) \geq 1-\epsilon$.
\end{defn}

Intuitively, an \eshadow\ is a (non-unique) region that contains the obstacle with probability at least $1-\epsilon$. A consequence of this definition is that if there exists an \eshadow\ of $\mathcal{O}$ that does not intersect the robot, then that \eshadow\ provides a certificate that the probability of collision between the robot and $\mathcal{O}$ is no more than $\epsilon$, since $P(\mathcal{O} \nsubseteq S) \leq \epsilon$:
\begin{equation}
    \bigcup_i \mathcal{S} \cap \mathcal{X}_i = \emptyset \implies P\pn{\bigcup_i \mathcal{X}_i \cap \mathcal{O} \neq \emptyset} \leq \epsilon \label{shadow_coll}
\end{equation}

To preclude trivial examples, such as $\mathcal{S} = \R^n$, we follow Axelrod et al. in considering only maximal \eshadow s:

\begin{defn}\label{max_eshadow_defn}
    {\normalfont (maximal \eshadow)} A set $\mathcal{S} \subseteq \R^n$ is a maximal \eshadow \ of $\mathcal{O}$ if the probability $P(\mathcal{O} \subseteq \mathcal{S}) = 1-\epsilon$.
\end{defn}

This refinement is important from the point of view of certifying collision risk bounds: while a $0.1$-shadow is also an $0.5$-shadow, a maximal $0.1$-shadow certifies a much tighter bound on collision probability than does a maximal $0.5$-shadow. For the remainder of this paper, we restrict our attention to \textit{maximal} \eshadow s.

In their original treatment, Axelrod et al. model uncertain obstacles as polytopes with faces defined by linear inequalities with Gaussian uncertainty in the parameters (i.e. polytopes with Gaussian-distributed faces, or PGDFs). Jasour developed a similar approach in the case of polynomials with uncertain parameters \cite{jasourMomentSumofSquaresApproachFast2018}. Although these representations are very general, and although these \eshadow s can be computed easily as conic sections (in the PGDF case) or polynomials (in the polynomial case), they are not convex in general, as can be seen in Fig.~\ref{fig:shadows}a. While it is straightforward to test whether a point robot intersects one of these \eshadow s, this representation makes it difficult to generalize to non-trivial robot geometry by leveraging existing algorithms for fast convex-convex collision checking.

\begin{figure}[htbp]
    \centerline{\includegraphics[width=0.7\linewidth]{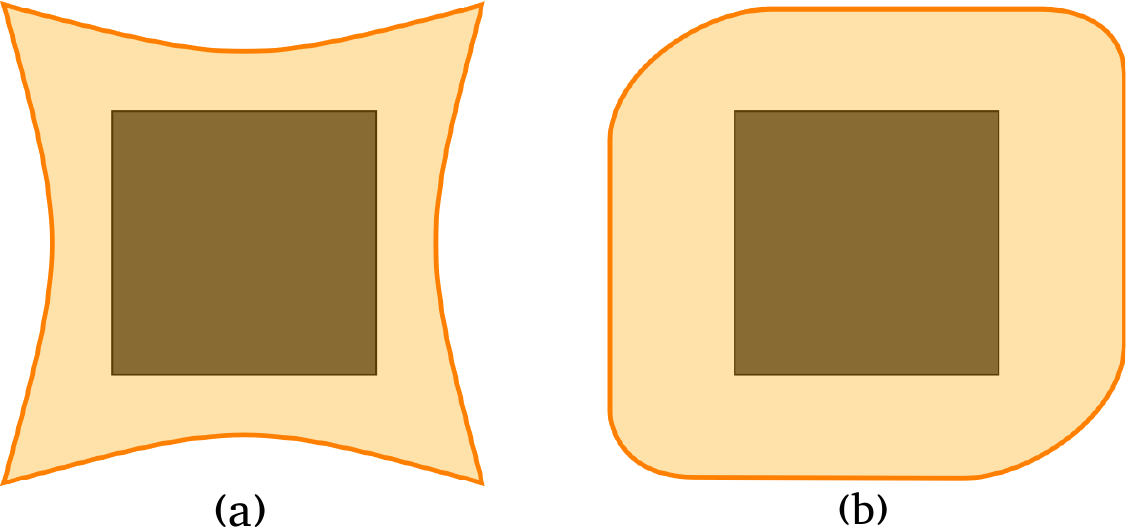}}
    \caption{(a) A representative \eshadow\ for a PGDF obstacle, as discussed in \cite{axelrodProvablySafeRobot2018}. These \eshadow s are non-convex in general. (b) A representative \eshadow\ constructed using Theorem~\ref{convexity_proof}, which is guaranteed to be convex.}
    \label{fig:shadows}
\end{figure}

Furthermore, although the PGDF representation of uncertain objects is a natural representation for obstacles derived from point cloud data, in many applications the size and shape of obstacles are known \textit{a priori} and it is only the location of the obstacle that is uncertain. In a factory, a robot might have a 3D model of a forklift but be uncertain of its exact location. Especially when objects are tracked using a computer-vision system, it is natural to represent an uncertain object as a known rigid body with uncertain 3D pose, since a forklift might be accurately identified even though its location is uncertain.

In our approach, we limit the uncertainty to affect only the location (and not the orientation) of obstacles in 3D space. This assumption restricts the range of uncertainty that we can model, but it allows us to guarantee that the resulting \eshadow s are convex as long as the underlying obstacle is convex (as shown in the following constructive proof).

\begin{thm}\label{convexity_proof}
    If an obstacle $\mathcal{O}$ with known convex geometry is subject to additive Gaussian uncertainty in its position, then there exists a convex maximal \eshadow\ of $\mathcal{O}$ for any $\epsilon > 0$.
\end{thm}
\begin{proof}
    As defined above, let $\mathcal{O}$ denote the set of points occupied by the uncertain obstacle, and let $O$ denote the (convex) set of points occupied by the nominal geometry of that obstacle (located at the expected location of the uncertain object, so that the additive Gaussian noise is zero-mean). We can express this relationship formally as
    \begin{equation}
        \mathcal{O} = \set{x + d : x \in O}; \quad d \sim \mathcal{N}(0, \Sigma)
    \end{equation}
    where $\mathcal{N}(0, \Sigma)$ is the multivariate Gaussian distribution with zero mean and covariance $\Sigma$. Since $d$ is a zero-mean Gaussian random variable, then if we define the set
    \begin{equation}
        \mathcal{D} = \set{d : d^T \Sigma^{-1} d \leq \phi^{-1}(1-\epsilon)}
    \end{equation}
    where $\phi^{-1}$ is the inverse of the cumulative distribution function (CDF) of the chi-squared distribution with $n$ degrees of freedom, then $P(d \in \mathcal{D}) = 1-\epsilon$ (as shown in \cite{axelrodProvablySafeRobot2018}). Next, we can define the Minkowski sum $\mathcal{S} = \mathcal{O}+\mathcal{D}$ as
    \begin{equation}
        \mathcal{S} = \set{x+d: x \in O, d \in \mathcal{D}}
    \end{equation}
    If we denote the probability that $\mathcal{O}$ is a subset of $\mathcal{S}$ as $P(\mathcal{O} \subseteq \mathcal{S})$, then we can observe that this event is equivalent to the event that for all $y \in \mathcal{O}$ there exists an $x \in O$ and $d \in \mathcal{D}$ such that $x+d=y$.
    %
    %
    By definition, every $y \in \mathcal{O}$ equals $x_y + d_y$ for some $x_y \in O$ and $d_y \sim \mathcal{N}(0, \Sigma)$,
    so this last event reduces to $d_y \in \mathcal{D};\ d_y \sim \mathcal{N}(0, \Sigma)$.
    Because of these equivalencies, we see that
    \begin{equation}
        P(\mathcal{O} \subseteq \mathcal{S}) = P(d_y \in \mathcal{D}) = 1-\epsilon
    \end{equation}
    Thus, we see that $\mathcal{S}$ is a maximal \eshadow\ of $\mathcal{O}$ for any $\epsilon > 0$ (where $\phi^{-1}(1-\epsilon)$ is finite). Furthermore, it is straightforward to show that the Minkowski sum of two convex sets is itself a convex set. Let $x=x_O + x_D$ and $y=y_O + y_D$ be points in $\mathcal{S}$, where $x_O$ and $y_O$ are points in $O$ and $x_D$ and $y_D$ are points in $\mathcal{D}$, and let $\lambda \in [0, 1]$. Note that
    \begin{align}
        \lambda x + (1-\lambda)y &= \lambda(x_O + x_D) + (1-\lambda)(y_O + y_D) \\
        &= \lambda x_O + (1-\lambda)y_O + \lambda x_D + (1-\lambda)y_D
    \end{align}
    Both $O$ and $\mathcal{D}$ are convex, so $\lambda x_O + (1-\lambda)y_O \in \mathcal{O}$ and $\lambda x_D + (1-\lambda)y_D \in \mathcal{D}$. It follows that $\lambda x + (1-\lambda)y \in \mathcal{S}$, so $\mathcal{S}$ is both a maximal \eshadow\ of $\mathcal{O}$ and convex. An example \eshadow\ generated using this procedure is shown in Fig.~\ref{fig:shadows}b.
\end{proof}

We remark that since $\mathcal{S}$ is convex, we can apply techniques such as the GJK algorithm to check whether $\mathcal{S}$ intersects with another convex shape in linear time with the number of vertices involved \cite{gilbertFastProcedureComputing1988}. Moreover, most modern collision checking libraries improve on this performance by using a two-step collision checking approach, so that the GJK algorithm is only run on shapes that are close enough to conceivably intersect, avoiding wasting effort on obviously non-colliding pairs. As a result, in practice $\mathcal{S}$ can be checked for collision with other convex shapes quite quickly. 

Furthermore, although explicitly computing the Minkowski sum of two convex shapes is an expensive operation ($O(n^2)$ in the number of vertices in the two shapes \cite{parkEfficientProbabilisticCollision2017}), the GJK algorithm can be run without explicitly constructing these sums by representing convex shapes using a support mapping, which maps directions in $\R^n$ to the point in a shape furthest in that direction. Because the support of a Minkowski sum is simply the sum of the supports of the two shapes, the GJK algorithm can be executed using an implicit representation of the Minkowski sum \cite{gilbertFastProcedureComputing1988}. This is one advantage of our approach compared with that presented in \cite{parkEfficientProbabilisticCollision2017}, which computes an integral over the faces of a Minkowski sum, incurring the full cost of explicitly constructing the Minkowski sum. As a result, the approach in \cite{parkEfficientProbabilisticCollision2017} is more accurate but also several orders of magnitude slower than our approach. Of course, the desired trade-off between accuracy and computation time is context-dependent, and efficiently utilizing both fast and slow estimates presents an intriguing opportunity for future work.

\section{Collision Probability Calculation}\label{sec:algs}

This section presents our approach for using convex \eshadow s to efficiently calculate upper bounds on the probability that an uncertain obstacle collides with a robot. Recall that because of the definition of an \eshadow, if no robot link intersects the \ejshadow\ of obstacle $\mathcal{O}_j$, then the probability of collision with $\mathcal{O}_j$ is at most $\epsilon_j$. Since we desire a tight upper bound to avoid excessively conservative estimates, we can apply a bisection search to iteratively calculate the smallest $\epsilon_j$ (or equivalently, the largest \ejshadow) such that there is no collision between the \ejshadow\ $\mathcal{S}_j$ and the robot, similarly to Axelrod et al. \cite{axelrodProvablySafeRobot2018}. This method is described in Algorithm~\ref{alg:one-shot}.

\begin{algorithm}[tb]
\SetAlgoLined
\SetKwInOut{Input}{Input}
\Input{A set of robot links $\mathcal{X}_i$, obstacle $\mathcal{O}$, covariance matrix $\Sigma$, and precision tolerance $\epsilon_{tol}$}
\KwResult{$\epsilon$ such that the true risk of collision with $\mathcal{O}$ cannot exceed $\epsilon+\epsilon_{tol}/2$}
 $\epsilon_{l} \gets 0$, $\epsilon_{u} \gets 1$\;
 $\epsilon \gets (\epsilon_u + \epsilon_l) / 2$\;
 \While{$\epsilon_u - \epsilon_l > \epsilon_{tol}$}{
  Construct \eshadow\  of $\mathcal{O}$ according to Theorem~\ref{convexity_proof}\;
  \eIf{$\epsilon$\text{-shadow of }$\mathcal{O}$\text{ intersects any }$\mathcal{X}_i$}{
   $\epsilon_l \gets \epsilon$\;
   }{
   $\epsilon_u \gets \epsilon$\;
  }
  $\epsilon \gets (\epsilon_u + \epsilon_l) / 2$\;
 }
 \caption{One-shot bisection search method for computing an upper bound on the risk of collision between the robot and an obstacle.}\label{alg:one-shot}
\end{algorithm}

An advantage of this bisection-search method for calculating the maximal \eshadow\ is that it requires only $\log(1/\epsilon_{tol})$ queries to the collision checking algorithm, where $\epsilon_{tol}$ is the tolerance for error in the estimate of $\epsilon$. However, as we can see from Fig.~\ref{fig:conservative}, the upper bound provided by the \ejshadow\ can be extremely conservative, since this method treats any case in which $\mathcal{O}$ protrudes beyond its \eshadow\ as a risk of collision. Since there are many cases in which $\mathcal{O}$ can protrude beyond its \eshadow\ without endangering the robot (as shown in Fig.~\ref{fig:conservative}), a method (like that proposed in Axelrod et al.) using only a single line search can yield overly conservative results. As a result, although Algorithm~\ref{alg:one-shot} extends existing approaches to non-trivial geometries, it is fairly conservative and acts as a good baseline for our second approach.

\begin{figure}[htbp]
    \centerline{\includegraphics[width=0.6\linewidth]{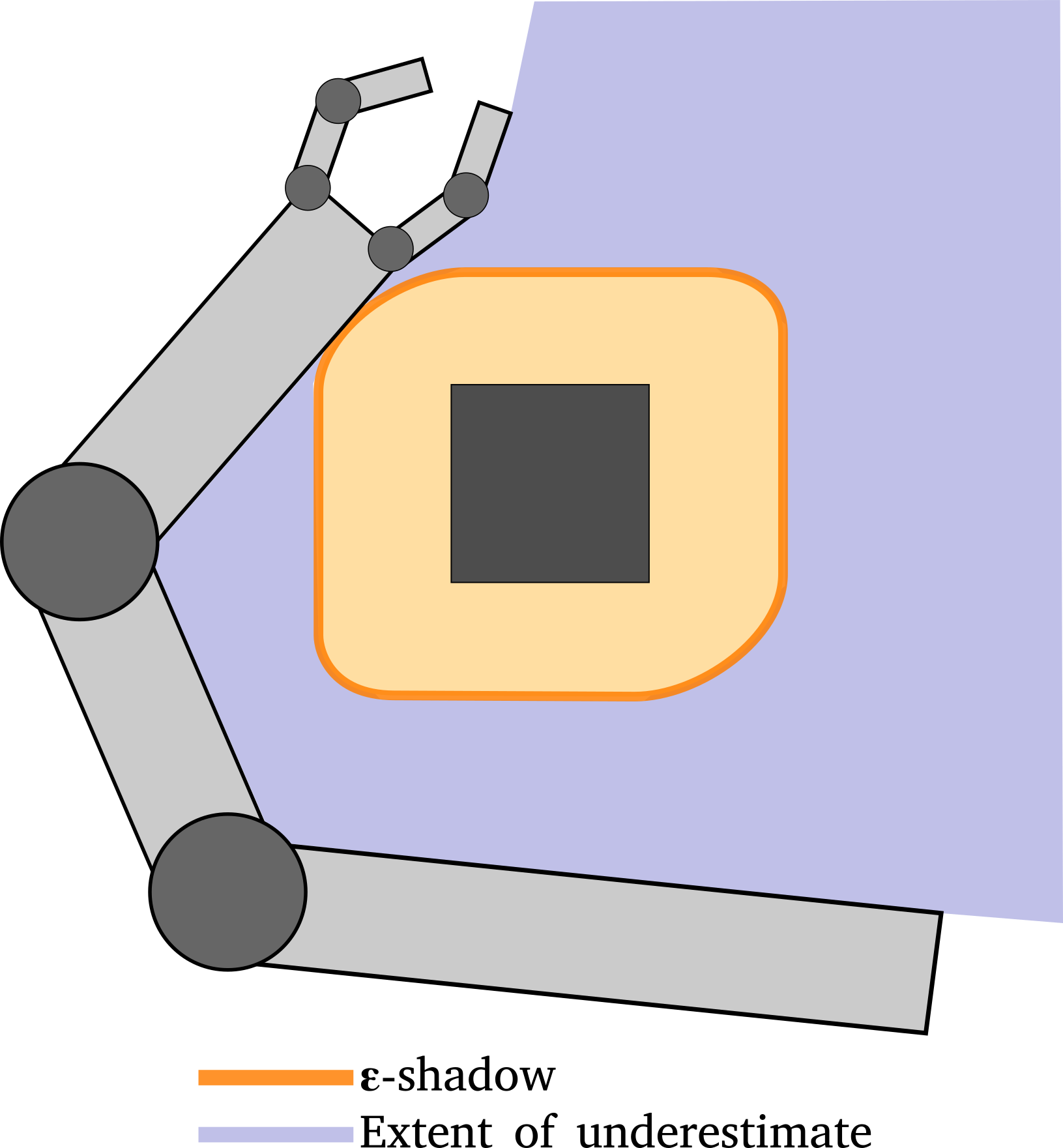}}
    \caption{A single bisection search (like that proposed in \cite{axelrodProvablySafeRobot2018}) can yield overly conservative estimates.}
    \label{fig:conservative}
\end{figure}

To produce a less conservative upper bound, we can exploit the fact that as we expand an obstacle's \eshadow, it is likely to collide first with only one link of the robot, denoted $\mathcal{X}_0$. As a result, there is often space around the \eshadow, away from $\mathcal{X}_0$, into which it can expand further without colliding with other links of the robot. The further the \eshadow\ can expand, the more collision risk it can capture, certifying a tighter bound on collision risk.
Pseudo-code for an algorithm taking advantage of this secondary expansion is provided in Algorithm~\ref{alg:two-shot} and illustrated in Fig.~\ref{fig:two-shot}. This ``two-shot'' algorithm provides tighter upper bounds on collision risk than those computed using a one-shot method, with only a minor trade-off in running time.

\begin{figure*}[htbp]
    \centerline{\includegraphics[width=0.3\linewidth]{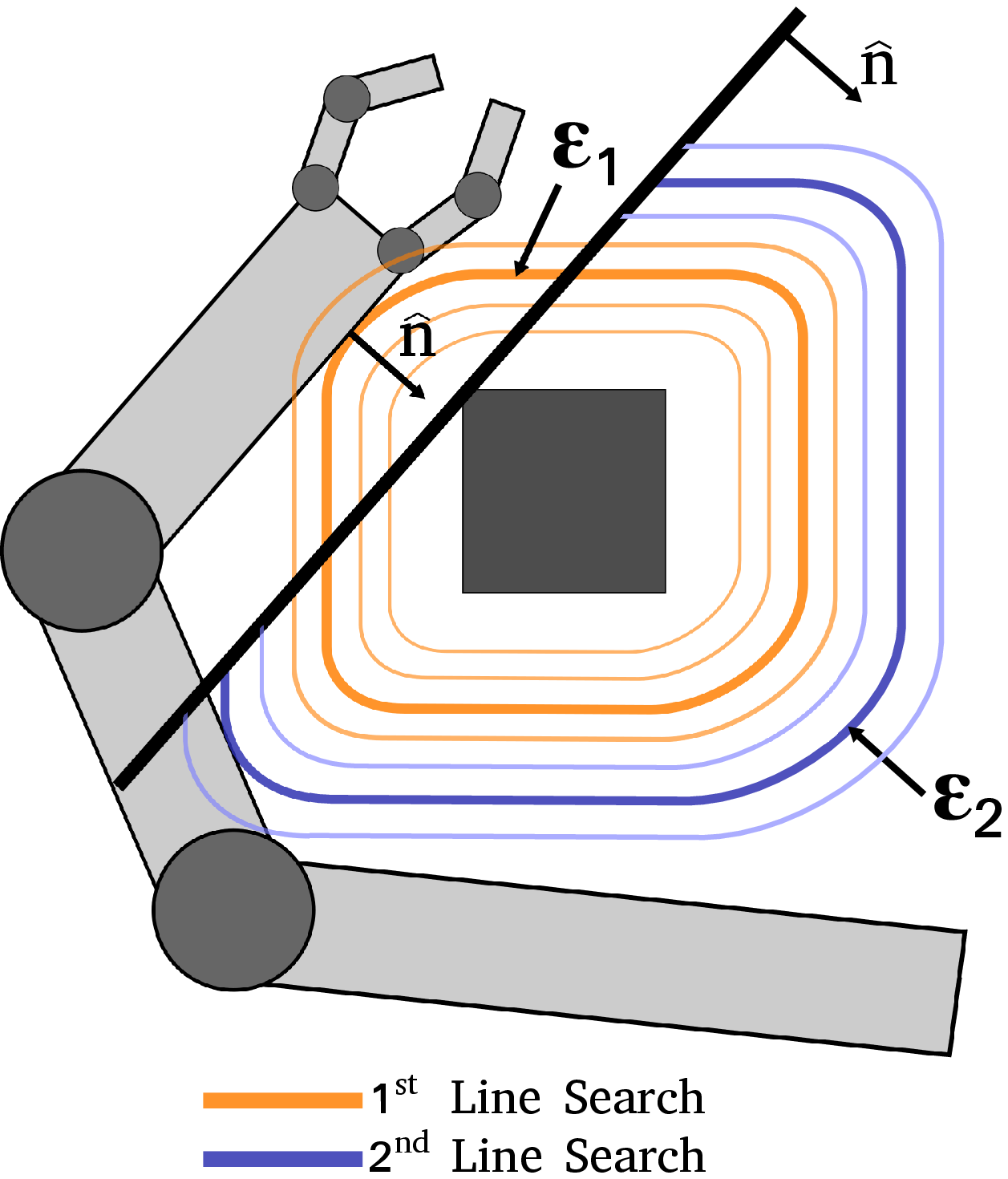}\hspace{1cm}\includegraphics[width=0.3\linewidth]{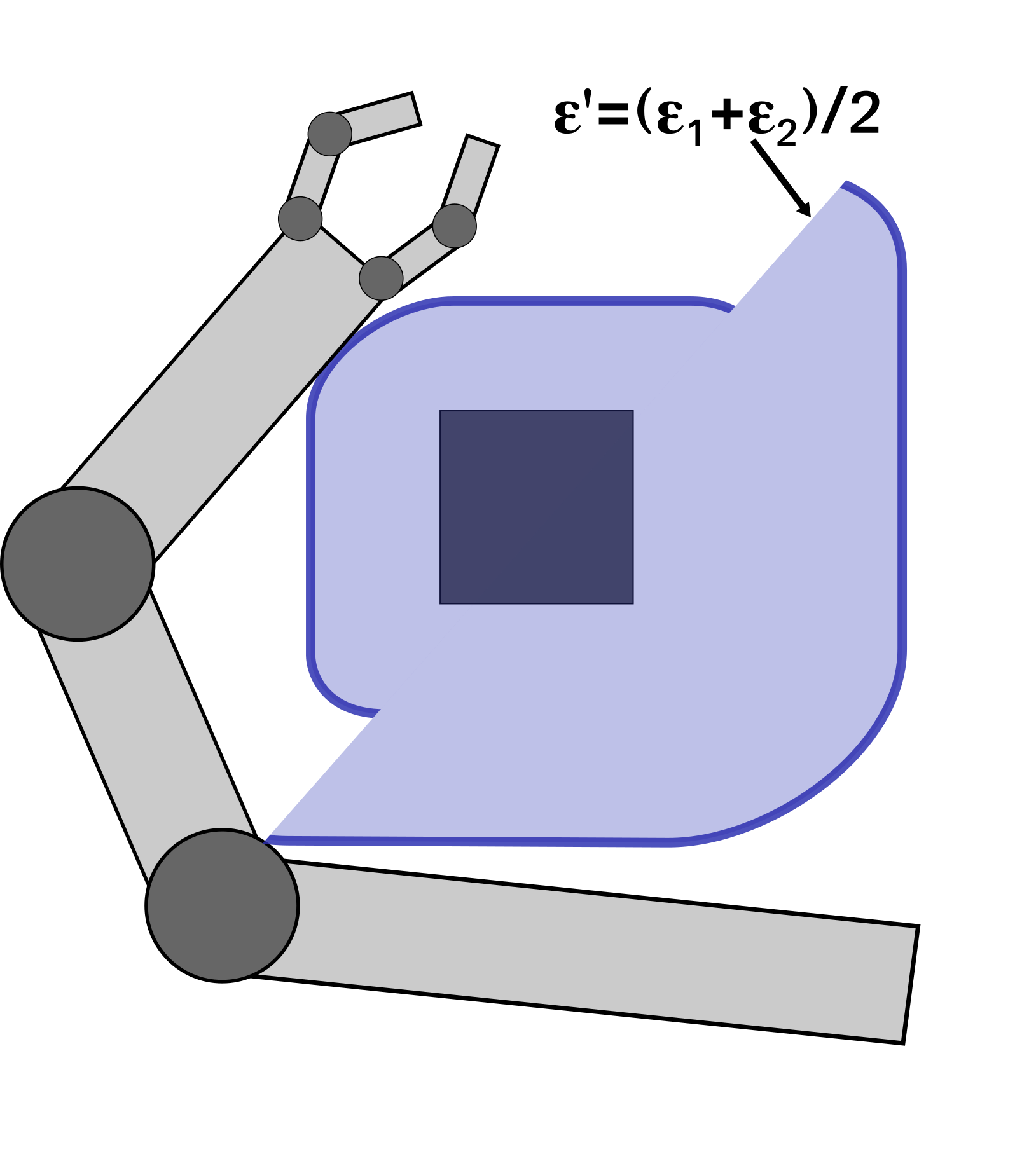}}
    \caption{An illustration of the two-shot algorithm for computing tighter upper bounds on collision risk. First we expand the \eshadow\ in all directions during the first line search, then we expand the \eshadow\ away from the first point of contact until a second contact occurs.}
    \label{fig:two-shot}
\end{figure*}

\begin{algorithm}[tb]
\SetAlgoLined
\SetKwInOut{Input}{Input}
\Input{A set of robot links $\mathcal{X}_i$, obstacle $\mathcal{O}$, covariance matrix $\Sigma$, and precision tolerance $\epsilon_{tol}$}
\KwResult{$\epsilon$ such that the true risk of collision with $\mathcal{O}$ cannot exceed $\epsilon+\epsilon_{tol}/2$}
 $\epsilon_{l} \gets 0$, $\epsilon_{u} \gets 1$\;
 $\epsilon_1 \gets (\epsilon_u + \epsilon_l) / 2$\;
 \While{$\epsilon_u - \epsilon_l > \epsilon_{tol}$}{
  Construct $\epsilon_1$-shadow of $\mathcal{O}$ according to Theorem~\ref{convexity_proof}\;
  \eIf{$\epsilon_1$\text{-shadow of }$\mathcal{O}$\text{ intersects any }$\mathcal{X}_i$}{
   $\epsilon_l \gets \epsilon$\;
   $\oldhat{n} \gets$ the normal between $\mathcal{O}$ and $\mathcal{X}_i$ at the point of collision, pointing into $\mathcal{O}$\;
   }{
   $\epsilon_u \gets \epsilon$\;
  }
  $\epsilon_1 \gets (\epsilon_u + \epsilon_l) / 2$\;
 }
 $\epsilon_{l} \gets 0$\;
 $\epsilon_2 \gets (\epsilon_u + \epsilon_l) / 2$\;
 \While{$\epsilon_u - \epsilon_l > \epsilon_{tol}$}{
  Construct $\epsilon_2$-shadow of $\mathcal{O}$ according to Theorem~\ref{two_shot_shadow_proof}\;
  \eIf{$\epsilon_2$\text{-shadow of }$\mathcal{O}$\text{ intersects any }$\mathcal{X}_i$}{
   $\epsilon_l \gets \epsilon$\;
   }{
   $\epsilon_u \gets \epsilon$\;
  }
  $\epsilon_2 \gets (\epsilon_u + \epsilon_l) / 2$\;
 }
 $\epsilon \gets (\epsilon_1 + \epsilon_2)/2$
 \caption{Two-shot method for certifying tighter bounds on robot-obstacle collision risk.}\label{alg:two-shot}
\end{algorithm}

Intuitively, this two-shot algorithm works by first finding the largest $\epsilon_1$-shadow (smallest $\epsilon_1$) that does not intersect with the robot but is tangent to it at some link $\mathcal{X}_{0}$. We then calculate the collision normal vector $\oldhat{n}$ at the interface of the $\epsilon_1$-shadow and $\mathcal{X}_{0}$ (pointing into the $\epsilon_1$-shadow) and apply a second bisection search to find a larger $\epsilon_2$-shadow that extends the $\epsilon_1$-shadow in the $\oldhat{n}$-direction. To expand in the $\oldhat{n}$-direction, this second search constructs the $\epsilon_2$-shadow as the Minkowski sum of the obstacle and an ellipsoid intersected with the half-space $\oldhat{n} \cdot x \geq 0$. This $\epsilon_2$-shadow is also convex and can thus be checked for collision quickly using existing algorithms. Furthermore, finding this second shadow requires no more than twice the number of collision checks needed by the single-shot line search and thus also runs in $O(\log{1/\epsilon_{tol}})$ time complexity.

Once these $\epsilon_1$- and $\epsilon_2$-shadows have been found, we can take the union of these shapes as an $\epsilon'$-shadow, where $\epsilon' = (\epsilon_1 + \epsilon_2)/2$. The following theorems formalize this approach.

\begin{thm}\label{two_shot_shadow_proof}
    Let $\mathcal{D} = \set{d : d^T\Sigma^{-1} d \leq \phi^{-1}(1-\epsilon_2)}$ be an ellipsoid and $\mathcal{D}' = \mathcal{D} \cap \set{\oldhat{n} \cdot d \geq 0}$
    be the intersection of that ellipsoid with a half-space. Moreover, let $\mathcal{S}_2$ be the Minkowski sum of $O$ and $\mathcal{D}'$. Then $\mathcal{S}_2$ is a maximal $\epsilon_2/2$-shadow of $\mathcal{O}$. Furthermore, $\mathcal{S}_2$ is convex.
\end{thm}
\begin{proof}
    Let $d \sim \mathcal{N}(0, \Sigma)$ be a zero-mean Gaussian random variable representing the uncertain displacement of obstacle $\mathcal{O}$ from its nominal position. The probability that $d$ falls within the half-ellipsoid $\mathcal{D}'$ is given by
    \begin{align}
        P(d \in \mathcal{D}') &= P(d \in \mathcal{D} \cap d \in \set{x : \oldhat{n}\cdot x \geq 0}) \\
        &= P(d \in \mathcal{D}) P(d \in \set{x : \oldhat{n}\cdot x \geq 0}) \\
        &= \frac{1}{2} P(d \in \mathcal{D})
    \end{align}
    due to the symmetry of the Gaussian distribution.
    Recall that $P(d \in \mathcal{D}) = \epsilon_2$, so $P(d \in \mathcal{D}') = \epsilon_2/2$.
    %
    Using the same reasoning as in Theorem~\ref{convexity_proof}, it follows that
    \begin{equation}
        P(\mathcal{O} \subseteq \mathcal{S}_2) = P(d \in \mathcal{D}') = \frac{\epsilon_2}{2}
    \end{equation}
    which is sufficient to show that $\mathcal{S}_{2}$ is a maximal $\epsilon_2/2$-shadow of $\mathcal{O}$. To complete the proof, we observe that the intersection of an ellipsoid and a half-space is convex, and the Minkowski sum of two convex shapes is convex, so $\mathcal{S}_2$ is convex.
\end{proof}

\begin{thm}\label{union_proof}
    Let $\mathcal{S}_1$ be a maximal $\epsilon_1$-shadow constructed according to Theorem~\ref{convexity_proof}
    %
    and let $\mathcal{S}_2$ be a maximal $\epsilon_2/2$-shadow constructed according to Theorem~\ref{two_shot_shadow_proof}.
    Then $\mathcal{S}' = \mathcal{S}_1 \cup \mathcal{S}_2$
    %
    %
    is a maximal $\epsilon'$-shadow of $\mathcal{O}$ with $\epsilon' = (\epsilon_1 + \epsilon_2)/2$.
\end{thm}
\begin{proof}
    Following the logic of our proof of Theorem~\ref{convexity_proof}, we see that the event $\mathcal{O} \subseteq \mathcal{S}'$ reduces to the disjunction
    \begin{equation}
        d \in \mathcal{D}_1 \vee d \in \mathcal{D}_2
    \end{equation}
    where $d \sim \mathcal{N}(0, \Sigma)$ is the uncertain displacement of obstacle $\mathcal{O}$ from its nominal position, and
    \begin{align}
        \mathcal{D}_1 &= \set{d : d^T\Sigma^{-1} d \leq \phi^{-1}(1-\epsilon_1)} \\
        \mathcal{D}_2 &= \set{d : d^T\Sigma^{-1} d \leq \phi^{-1}(1-\epsilon_2), \oldhat{n}\cdot d \geq 0}
    \end{align}
    Using this reduction, we see that
    \begin{align}
        P(\mathcal{O} \subseteq \mathcal{S}') &= P(d \in \mathcal{D}_1 \vee d \in \mathcal{D}_2) \\
        &= P(d \in \mathcal{D}_1) + P(d \in \mathcal{D}_2) - P(d \in \mathcal{D}_1 \cap \mathcal{D}_2) \\
        &= \epsilon_1 + \frac{\epsilon_2}{2} - P(d \in \mathcal{D}_1 \cap \mathcal{D}_2)
    \end{align}

    Note that the second line search in Algorithm~\ref{alg:two-shot} always yields $\epsilon_2 \leq \epsilon_1$ (since it searches outwards from the \eshadow\ found during the first line search), and recall that the inverse CDF of the chi-squared distribution $\phi^{-1}$ is strictly increasing. As a result, the intersection $\mathcal{D}_1 \cap \mathcal{D}_2$ can be expressed more succinctly as the set
    \begin{equation}
        \set{d : d^T \Sigma^{-1} d \leq \phi^{-1}(1-\epsilon_1);\ \oldhat{n}\cdot d \geq 0}
    \end{equation}
    From Theorem~\ref{two_shot_shadow_proof}, it follows that $P(\mathcal{D}_1 \cap \mathcal{D}_2) = \epsilon_1/2$,
    %
    %
    so we see that
    \begin{equation}
        P(\mathcal{O} \subseteq \mathcal{S}') = \epsilon' = \frac{\epsilon_1 + \epsilon_2}{2}
    \end{equation}
\end{proof}

Because $\epsilon_2 \leq \epsilon_1$, we see that $\epsilon' \leq \epsilon_1$, so the shadow constructed using this two-shot algorithm provides a tighter upper bound on the probability of collision than the single-shot algorithm based on Axelrod et al.. Furthermore, even though the enlarged shadow $\mathcal{S}'$ is not convex, our algorithm only requires collision checks involving $\mathcal{S}_1$ and $\mathcal{S}_2$, both of which are convex, preserving the performance benefits of our approach. Additionally, because the second bisection search in Algorithm~\ref{alg:two-shot} can be warm-started based on the results from the first, the cost of running the two-shot variant is often less than twice than that of the one-shot variant. The details of these performance trade-offs are made clear in the next section, where we characterize the performance of both algorithms.

\section{Performance}

Although the two-shot method discussed above provides a tighter bound on the probability of collision, there is necessarily a  trade-off between the increased accuracy of this estimate and the additional time needed to compute it. We implement this method using the Bullet collision checking library \cite{coumansBulletPhysicsEngine}, using the testing scenario shown in Fig.~\ref{fig:setup}, which includes a simplified manipulator with convex geometry in a scene with three obstacles. Each obstacle was assigned a qualitatively different covariance matrix. The location of the green cylinder obstacle has covariance matrix $0.01 I$, where $I$ is the $3\times3$ identity matrix, modeling equal uncertainty in all directions. The location of the yellow cylinder has covariance $$\Sigma_{yellow} = \mat{0.05, &0.07, &0.0\\0.07, &0.1, &0.0\\0.0, &0.0, &0.01}$$ modeling increased uncertainty in the $x$- and $y$-directions and relative certainty in the $z$-direction. The location of the red block has covariance $\Sigma_{red} = \text{diag}\mat{0.001, 0.001, 0.05}$, modeling relative certainty in the $x$- and $y$-directions and relative uncertainty in the $z$-direction. Note that all covariance matrices are expressed in the local frame of the obstacle. A covariance matrix in the global frame can be converted easily to one in the local obstacle frame by $\Sigma_{local} = R \pn{\Sigma_{global}} R^T$, where $R$ is the rotation matrix from the global frame to the local frame.

\begin{figure}[htbp]
    \centerline{\includegraphics[width=0.8\linewidth]{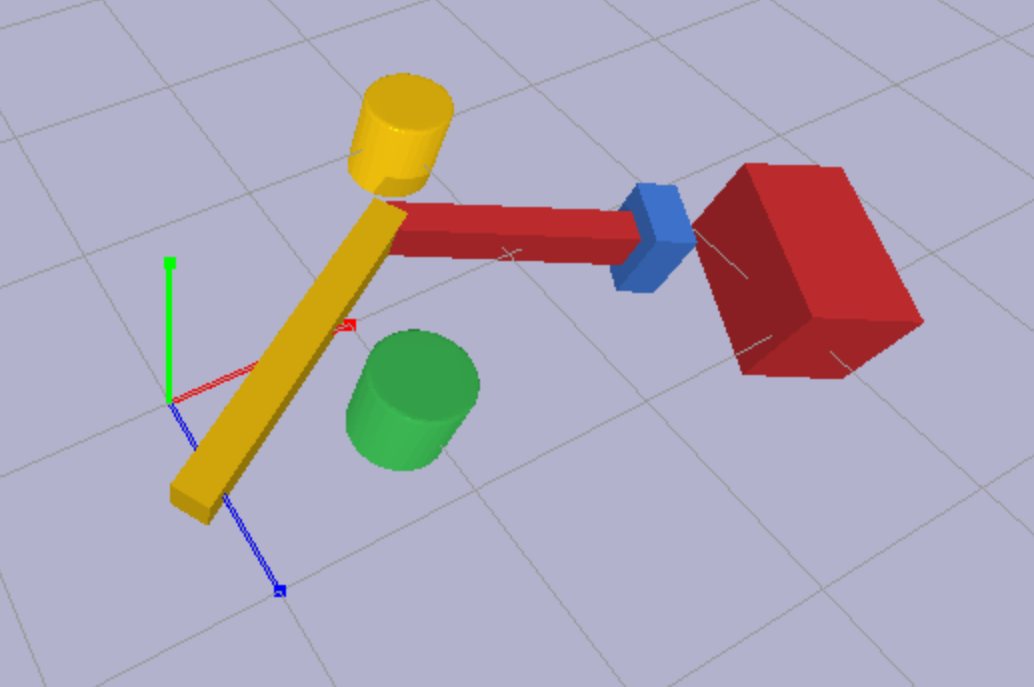}}
    \caption{A test scenario used to characterize the performance of our algorithms, showing a simplified manipulator in a scene with three unknown obstacles.}
    \label{fig:setup}
\end{figure}

\begin{figure}[htbp]
    \centerline{\includegraphics[width=0.8\linewidth]{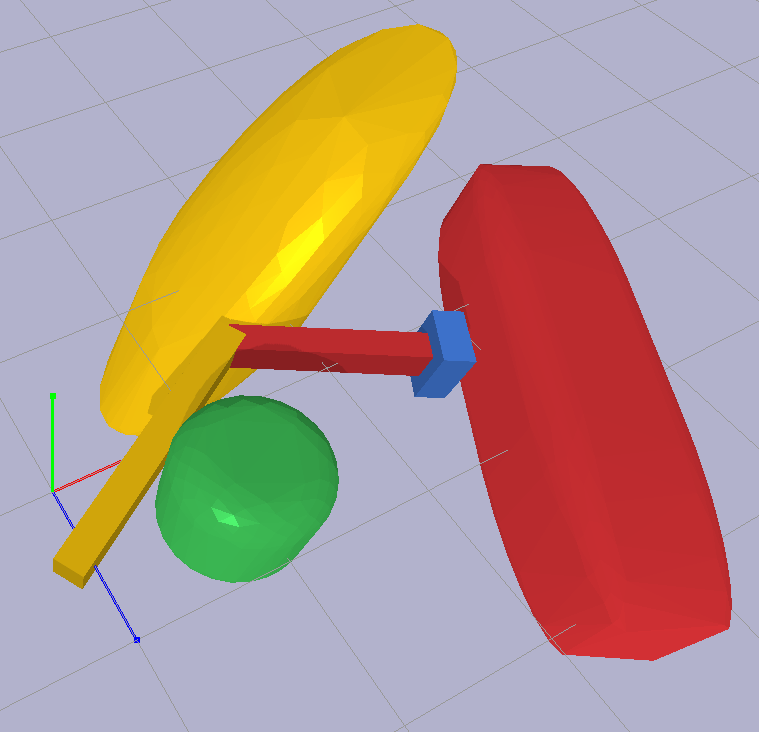}}
    \caption{The \eshadow s of the obstacles in Fig.~\ref{fig:setup}, computed using Algorithm~\ref{alg:one-shot}.}
    \label{fig:setup_shadows}
\end{figure}

\begin{figure}[htbp]
    \centerline{\includegraphics[width=0.8\linewidth]{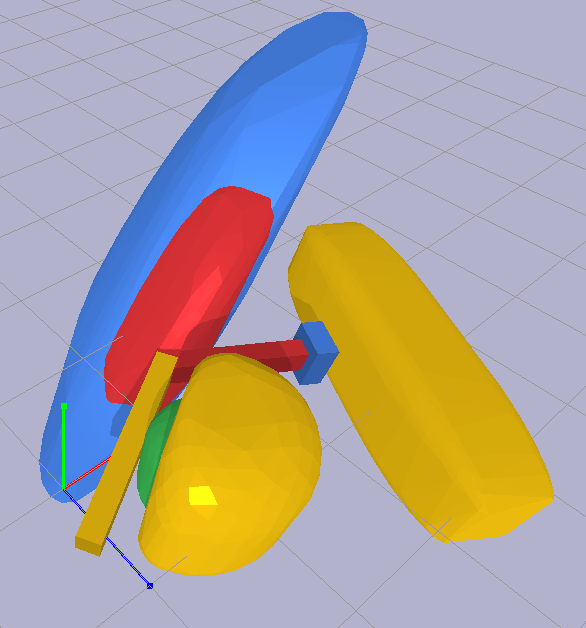}}
    \caption{The \eshadow s of the obstacles in Fig.~\ref{fig:setup}, computed using Algorithm~\ref{alg:two-shot}, providing less conservative probability bounds than those in Fig.~\ref{fig:setup_shadows}. In the case of the rightmost obstacle, the two-shot method yields no improvement as the estimated probability of collision has saturated at $0 \pm \epsilon_{tol}$.}
    \label{fig:setup_shadows_two_shot}
\end{figure}

In our implementation, we used the Bullet C++ library, running on one core of an Intel i9-7960X CPU, for performing deterministic convex-convex collision checking. To benchmark the performance of the underlying Bullet library, we performed $1,000,000$ collision checks between the robot and all three obstacles in Fig.~\ref{fig:setup}, randomly perturbing the position of each obstacle in each trial. On average, each obstacle required \SI{0.401}{\micro\second} per collision check. Since both Algorithms~\ref{alg:one-shot} and~\ref{alg:two-shot} require multiple calls to this collision checking subroutine, this figure provides a baseline for the performance of our algorithms. However, it is important to note that modern collision-checking libraries use different sub-solvers depending on the context, employing faster, less-accurate methods when objects are well separated, and devolving to more accurate routines when objects are in near-collision. As a result, the time required for collision checking varies: checks for obstacles in near-collision are slower, and these near-collision cases necessarily account for the majority of bisection search iterations.

The performance of our one- and two-shot algorithms on the test scenario in Fig.~\ref{fig:setup} is shown in Table~\ref{tab:test_scenario_results}. As expected, the two-shot algorithm requires more calls to the collision-checker, and thus runs more slowly, but it provides a much tighter bound on the probability of collision. In particular, the two-shot method provides a bound that is twice as tight as the one-shot method; this behavior is typical when all robot links are positioned to one side of an obstacle.

\begin{table}
\caption{Performance of the proposed collision probability estimation algorithms on the test scenario in Fig.~\ref{fig:setup}} \label{tab:test_scenario_results}
\begin{center}
    \begin{tabular}{|c|c|c|c|}
        \hline
        \textit{Algorithm}& \textit{Avg. run-time} & \textit{Est. collision} & \textit{True collision} \\
         & \textit{per obstacle (\SI{}{\micro\second})$^a$} & \textit{probability $\epsilon${}$^b$} & \textit{probability $\epsilon_0${}$^c$} \\
        \hline
        One-shot & 91.12 & 0.161005 & 0.011848 \\
        Two-shot & 157.74 & 0.080503 & 0.011848 \\
        \hline
        \multicolumn{4}{l}{$^a$ Averaged over $100,000$ trials.\ \ \  $^b$ Computed to tolerance $\epsilon_{tol} = 10^{-6}$.} \\
        \multicolumn{4}{l}{$^c$ Averaged over $1,000,000$ trials.}
    \end{tabular}
\label{tab1}
\end{center}
\end{table}

Both variants of our algorithm are theoretically guaranteed to provide upper bounds on the probability of collision. We can verify this guarantee empirically by varying the uncertainty associated with the obstacles in Fig.~\ref{fig:setup} (simply by scaling the relevant covariance matrices) and comparing the estimated probability bounds to the true probability of collision (calculated as an average over 1,000,000 trials). This comparison is shown in Fig.~\ref{fig:bounds_abs}, where we see that the two estimates indeed provide an upper bound on the true probability of collision. Moreover, we also see in Fig.~\ref{fig:bounds_abs} that the two-shot estimate provides a tighter upper bound than the one-shot estimate; indeed, the two-shot estimate is guaranteed to be no greater than the one-shot estimate, and except in degenerate cases the two-shot estimate is strictly tighter.

\begin{figure}
    \centerline{\includegraphics[width=\linewidth]{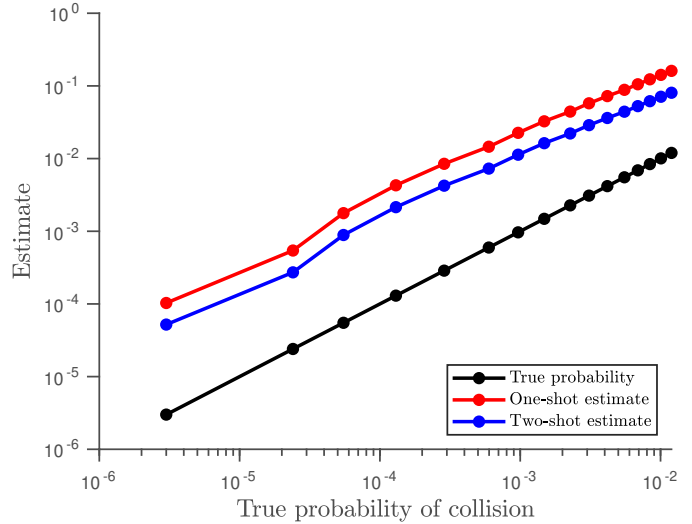}}
    \caption{Estimated collision probability vs. the true probability of collision (calculated from $1,000,000$ samples).}
    \label{fig:bounds_abs}
\end{figure}

We can also examine how the relative accuracy of these estimates changes as the true probability of collision varies. These results, shown in Fig.~\ref{fig:bounds_rel}, show that although the estimated bounds become looser as the true probability of collision decreases, the two-shot variant performs better than the one-shot variant, especially as the probability of collision goes to zero. Performance in this low-probability regime is particularly important in many applications, such as when robots are working in close proximity with humans or when autonomous vehicles plan motions around other cars, where users require that collisions are low-probability events.

\begin{figure}
    \centerline{\includegraphics[width=0.9\linewidth]{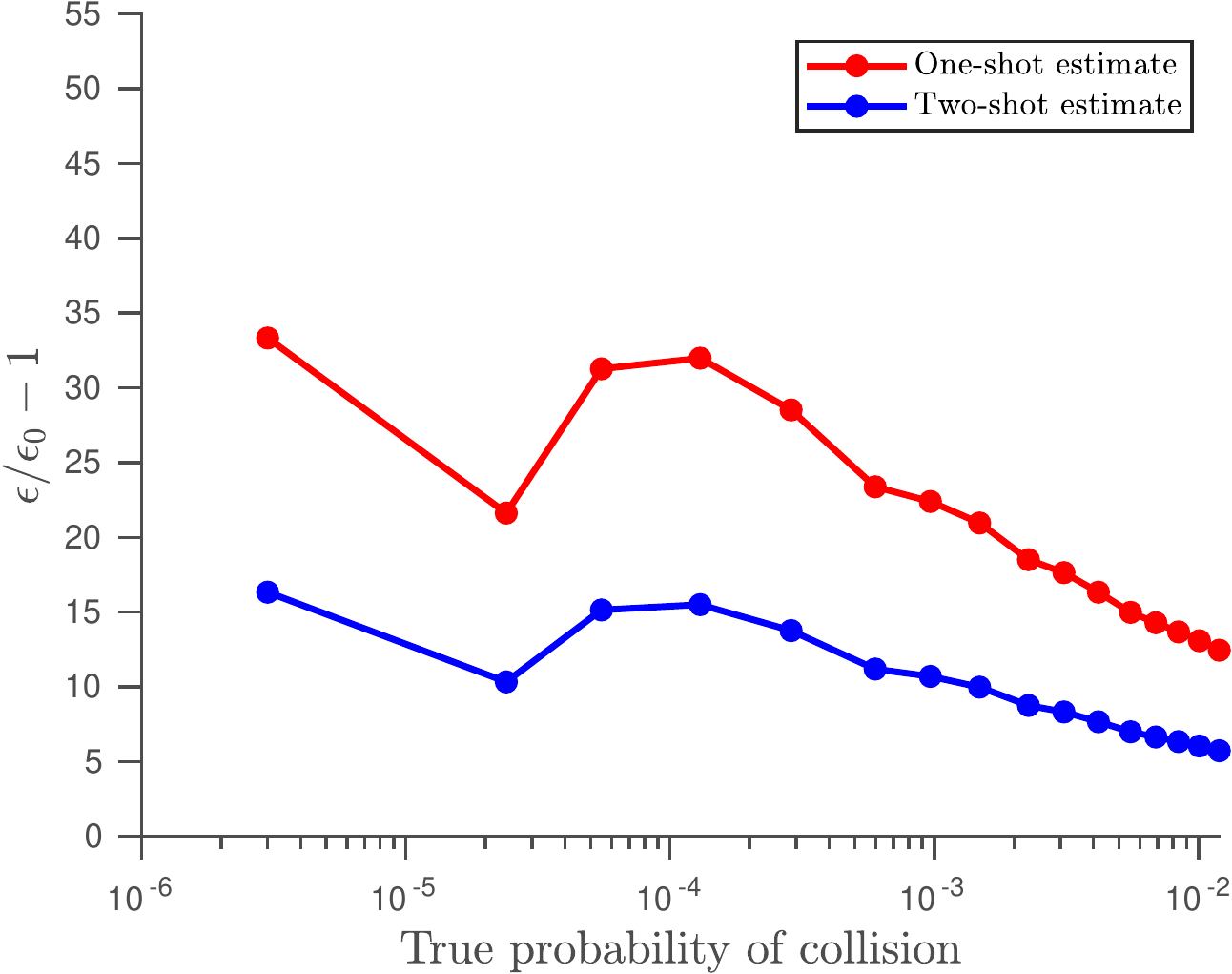}}
    \caption{Relative error in estimated probability of collision, defined as $\epsilon/\epsilon_0 - 1$, where $\epsilon_0$ is the true probability of collision.}
    \label{fig:bounds_rel}
\end{figure}

Finally, it is important to assess how our approach scales to more complicated environments. To evaluate how performance scales when we add additional robot links and obstacles to the scene, we used scenarios like those pictured in Fig.~\ref{fig:multi_setup}, where the central cube is an obstacle and the number of surrounding robot links varies (or vice-versa, with a central robot link and a variable number of surrounding obstacles). Because of the symmetry in this environment, the underlying geometry routines cannot accelerate collision checking by disqualifying obviously far-away candidates, making these scenarios close to the worst case. To evaluate the performance as $\epsilon_{tol}$ is varied, we used the example scenario shown in Fig.~\ref{fig:setup}. Our results, expressed as running time per obstacle averaged over $100,000$ trials, are shown in Fig.~\ref{fig:runtime}.

At a high level, we observe that adding additional robot links causes the time needed for individual queries to increase linearly. Since each query involves only a single obstacle, we would not expect adding obstacles to the scene to change the query time significantly, and this expectation is confirmed by the data. Additionally, we note that the running time varies logarithmically with the specified tolerance, as expected from an bisection-search based algorithm.

We note that even with demanding precision tolerances, our method consistently runs in under \SI{200}{\micro s}. Because of this low query time, our method can be used readily in real-time or trajectory-optimization applications, where collision probability estimation is a rate-limiting step.

\begin{figure}[htbp]
    \centerline{\includegraphics[width=0.8\linewidth]{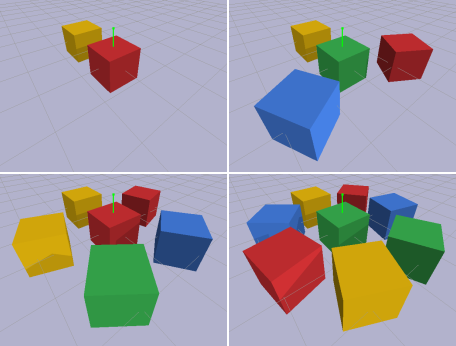}}
    \caption{Representative scenarios used to test the sensitivity of our algorithms' performance to the density of objects or robot links in the scene. The cube in the center represents an obstacle.}
    \label{fig:multi_setup}
\end{figure}

\begin{figure}
    \centerline{\includegraphics[width=0.9\linewidth]{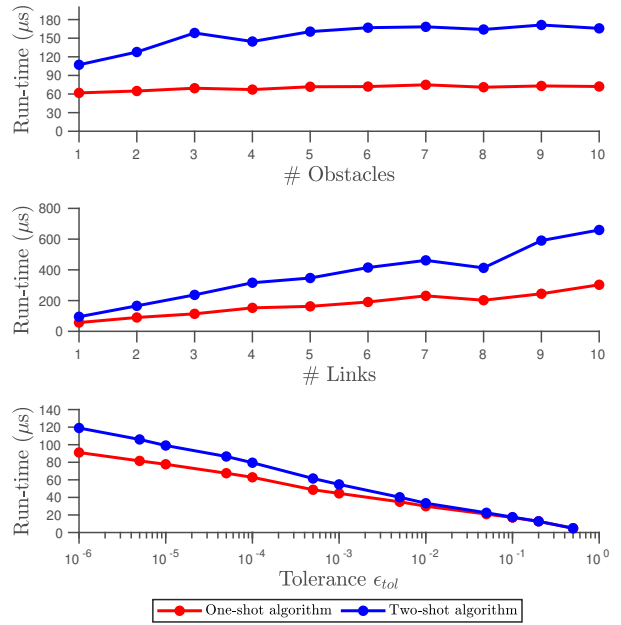}}
    \caption{Effect of scene complexity and tolerance $\epsilon_{tol}$ on performance. All run-times are reported on a per-obstacle basis.}
    \label{fig:runtime}
\end{figure}

\section{Conclusions}

In this paper, we present two algorithms for computing upper bounds on the probability of collision between uncertain convex objects, including a novel two-step algorithm that generates certificates of these bounds. Both algorithms provide strong theoretical guarantees that the true probability of collision does not exceed the estimate, and both provide extremely fast query times.

In future work, we intend to integrate this approach into an optimization framework to enable risk-constrained trajectory optimization in complex environments. Other exiting directions for future research include extending this approach to account for uncertainty in robot state as well as environment state and incorporating GPU-accelerated collision checking. 

We expect that our approach's combination of strong theoretical guarantees with fast performance will enable new uncertainty-aware trajectory planning and real-time risk monitoring algorithms, helping robots safely navigate inherently uncertain real-world environments.

\section*{Acknowledgments}

This work was supported by Airbus SE.

\bibliographystyle{IEEEtran}
\bibliography{ProximityAlertPaper}

\end{document}